\documentclass{article}

\usepackage{etoolbox}
\makeatletter
\patchcmd\@combinedblfloats{\box\@outputbox}{\unvbox\@outputbox}{}{%
}%
 \makeatother

\usepackage{microtype}
\usepackage{graphicx}
\usepackage{subfigure}
\usepackage{booktabs} 

\usepackage{hyperref}

\usepackage[accepted]{icml2021}

\icmltitlerunning{SGLB: Stochastic Gradient Langevin Boosting}

\usepackage{amsmath,amsfonts,bbold}

\DeclareMathOperator*{\argmin}{arg\,min}
\usepackage{amsthm}
\usepackage{balance}

\newtheorem{statement}{Statement}
\newtheorem{corollary}{Corollary}

\newtheorem{lemma}{Lemma}
\newtheorem{observation*}{Observation}
\newtheorem{theorem}{Theorem}

\usepackage{color}

\usepackage{wrapfig}

\begin{document}

\twocolumn[
\icmltitle{SGLB: Stochastic Gradient Langevin Boosting}



\icmlsetsymbol{equal}{*}

\begin{icmlauthorlist}
\icmlauthor{Aleksei Ustimenko}{yandex}
\icmlauthor{Liudmila Prokhorenkova}{yandex,mipt,hse}
\end{icmlauthorlist}

\icmlaffiliation{yandex}{Yandex, Moscow, Russia}
\icmlaffiliation{mipt}{Moscow Institute of Physics and Technology, Moscow, Russia}
\icmlaffiliation{hse}{HSE University, Moscow, Russia}

\icmlcorrespondingauthor{Aleksei Ustimenko}{austimenko@yandex-team.ru}

\icmlkeywords{Machine Learning, ICML}

\vskip 0.3in
]



\printAffiliationsAndNotice{}  

\begin{abstract}
This paper introduces Stochastic Gradient Langevin Boosting (SGLB)~--- a powerful and efficient machine learning framework that may deal with a wide range of loss functions and has provable generalization guarantees.
The method is based on a special form of the Langevin diffusion equation specifically designed for gradient boosting. 
This allows us to theoretically guarantee the \textit{global} convergence even for multimodal loss functions, while standard gradient boosting algorithms can guarantee only local optimum.
We also empirically show that SGLB outperforms classic gradient boosting when applied to classification tasks with 0-1 loss function, which is known to be multimodal.
\end{abstract}

\section{Introduction}

Gradient boosting is a powerful machine-learning method that iteratively combines weak models to obtain more accurate ones~\citep{friedman2001}. Nowadays, this technique remains the primary method for web search, recommendation systems, weather forecasting, and many other problems with complex dependencies and heterogeneous data.
Combined with decision trees, gradient boosting underlies such widely-used software libraries like XGBoost~\citep{Chen:2016}, LightGBM~\citep{LightGBM}, and CatBoost~\citep{catboost}.

For convex loss functions and under some regularity assumptions, stochastic gradient boosting (SGB) converges to the optimal solution~\cite{boulevard}. However, even local optima cannot be guaranteed for general losses. We fill this gap and build the first \textit{globally convergent} gradient boosting algorithm for convex and non-convex optimization with provable generalization guarantees. For this purpose, we combine gradient boosting with stochastic gradient Langevin dynamics (SGLD), which is a powerful iterative optimization algorithm~\citep{DBLP:journals/corr/RaginskyRT17}. It turns out that gradient boosting can be easily modified to a globally converging method: at each step, one has to shrink the currently built model and add a proper noise to stochastic gradient estimates and to the weak learners' selection algorithm.
To prove the global convergence and generalization, we develop a novel theoretical framework and show that the dynamics of SGLB can be approximated by a special form of the Langevin diffusion equation.

The proposed algorithm is implemented within the CatBoost open-source gradient boosting library (option \emph{langevin=True})~\cite{CatBoost_code}. Our experiments on synthetic and real datasets show that SGLB outperforms standard SGB and can optimize globally such non-convex losses as 0-1 loss, which was previously claimed to be a challenge~\citep{nguyen2013algorithms}. Since the first version of this paper appeared online, SGLB has found its new exciting applications in learning to rank and uncertainty estimation~\cite{ustimenko2020stochasticrank,malinin2021uncertainty}, see Section~\ref{sec:conclusion} for more details.

In the next section, we briefly discuss the related research on gradient boosting convergence and 0-1 loss optimization. Then, in Section~\ref{sec:background}, we give the necessary background on gradient boosting and stochastic gradient Langevin dynamics. The proposed SLGB method is described in Algorithm~\ref{alg:sglb} in Section~\ref{sec:algorithm}. This approach is easy to implement, and it is backed by strong theoretical guarantees. Our main results on the convergence of SGLB are given in Section~\ref{sec:results}: convergence on the train set is given in Theorem~\ref{thm:convergence}, the generalization error is bounded in Theorem~\ref{thm:generalization}, and our results are summarized in Corollary~\ref{corollary}. The experiments comparing SGLB with SGB on synthetic and real datasets are discussed in Section~\ref{sec:experiments}. Section~\ref{sec:conclusion} concludes the paper, discusses further applications of SGLB, and outlines promising directions for future research.

\section{Related Work}\label{sec:related}

In this section, we discuss related work on SGB convergence and 0-1 loss optimization. Our work is also closely related to stochastic gradient Langevin dynamics, which we discuss in Section~\ref{sec:SGLD}.

\paragraph{Convergence of gradient boosting}

There are several theoretical attempts to study SGB convergence, e.g., Boulevard~\citep{boulevard}, AnyBoost~\citep{anyboost}, or gradient boosting in general $L_2$ setting~\cite{doi:10.1198/016214503000125,zhang2005,2008arXiv0804.2752B,dombry2020behavior}. 
These works consider a general boosting algorithm, but under restrictive assumptions like exact greediness of the weak learners' selection algorithm~\citep{anyboost}, Structure Value Isolation properties~\citep{boulevard}, and, most importantly, convexity of the loss function. However, many practical tasks involve non-convex losses like 0-1 loss optimization~\citep{nguyen2013algorithms}, regret minimization in non-convex games \cite{DBLP:journals/corr/abs-1708-00075}, learning to rank~\cite{Liu:2009}, learning to select with order~\citep{vorobev2019learning}, and many others. Thus, existing frameworks fail to solve these tasks as they strongly rely on convexity. 

Many practical implementations of boosting like XGBoost~\citep{Chen:2016}, LightGBM~\citep{LightGBM}, and CatBoost~\citep{catboost} use constant learning rate in their default settings as in practice it outperforms dynamically decreasing ones. However, existing works on the convergence of boosting algorithms assume decreasing learning rates~\cite{zhang2005,boulevard}, thus leaving an open question: if we assume constant learning rate $\epsilon > 0$, can convergence be guaranteed? 

\paragraph{0-1 loss optimization}

For binary classification problems, convex loss functions are usually used since they can be efficiently optimized. However, as pointed out by~\citet{nguyen2013algorithms}, such losses are sensitive to outliers. On the other hand, 0-1 loss (the fraction of incorrectly predicted labels) is more robust and more interpretable but harder to optimize. \citet{nguyen2013algorithms} propose smoothing 0-1 loss with sigmoid function and show that an iteratively unrelaxed coordinate descent approach for gradient optimization of this smoothed loss outperforms optimization of convex upper bounds of the original 0-1 loss. In the current paper, we use a smoothed 0-1 loss as an example of a multimodal function and show that the SGLB algorithm achieves superior performance for this loss.

\section{Background}\label{sec:background}

\subsection{General Setup}\label{sec:setup}

Assume that we are given a distribution $\mathcal{D}$ on $\mathcal{X}\times \mathcal{Y}$, where $\mathcal{X}$ is a feature space (typically $\mathbb{R}^k$) and $\mathcal{Y}$ is a target space (typically $\mathbb{R}$ for regression or $\{0, 1\}$ for classification).\footnote{Table~\ref{tab:notation} in Appendix lists notation frequently used in the paper.}
We are also given a loss function $L(z, y):\mathcal{Z}\times\mathcal{Y}\rightarrow \mathbb{R}$, where $\mathcal{Z}$ is a space of predictions (typically $\mathbb{R}$ or $\{0,1\}$). Our goal is to minimize the expected loss $\mathcal{L}(f|\mathcal{D}) := \mathbb{E}_{\mathcal{D}} L(f(x), y)$  
over functions $f$ belonging to some family $\mathcal{F} \subset \{f:\mathcal{X}\rightarrow\mathcal{Z}\}$. In practice, the distribution $\mathcal{D}$ is unknown and we are given i.i.d.~samples $(x_1, y_1), \ldots, (x_N, y_N) \sim \mathcal{D}$ denoted as $\mathcal{D}_N$, so the expected loss is replaced by the empirical average:
\begin{equation}
\mathcal{L}_N(f) := \mathcal{L}(f|\mathcal{D}_N) = \frac{1}{N}\sum_{i=1}^N L(f(x_i), y_i)\,.
\end{equation}
Typically, a regularization term is added to improve generalization:
\begin{equation}\label{eq:reg_loss}
    \mathcal{L}_N(f,\gamma) := \mathcal{L}_N(f) + \frac{\gamma}{2} \|f\|_2^2\,.
\end{equation}

We consider only $\mathcal{F}$ corresponding to additive ensembles of weak learners $\mathcal{H} := \{h^s(x,\theta^s):\mathcal{X}\times \mathbb{R}^{m_s}\rightarrow \mathbb{R}, s\in S\}$, where $S$ is some finite index set and $h^s$ depends linearly on its parameters $\theta^s$.
Linearity in $\theta^s$ is crucial for convergence guarantees and is typical for practical implementations, where decision trees are used as base learners. 

A decision tree is a model built by a recursive partition of the feature space into disjoint regions called leaves.\footnote{Usually, a tree is constructed in a top-down greedy manner. We assume an arbitrary tree construction procedure satisfying mild assumptions formalized in Section~\ref{sec:results}.} Each leaf is assigned to a value that estimates the response $y$ in the corresponding region. Denoting these regions by $R_j$, we can write $h(x,\theta) = \sum_j \theta_j \mathbb{1}_{\{x \in R_j\}}$. Thus, if we know the tree structure given by $R_j$, a decision tree is a linear function of leaf values $\theta$. The finiteness of $S$ is a natural assumption: e.g., for decision trees, the tree depth and the number of splits for each feature are usually limited by fixed values, so there is a finite number of possible trees.
Further in the paper, we denote by $H_{s}: \mathbb{R}^{m_s} \to \mathbb{R}^N$ a linear operator converting $\theta^{s}$ to $(h^{s}(x_i,\theta^{s}))_{i=1}^N$.

Due to linear dependence of $h^s$ on $\theta^s$ and finiteness of $S$, we can represent any ensemble of models from $\mathcal{H}$ as a linear model $f_\Theta(x) = \langle\phi(x), \Theta\rangle_2$ for some feature map $\phi(x):\mathcal{X}\rightarrow \mathbb{R}^m$. Here $\Theta\in \mathbb{R}^m$ is a vector of parameters of the whole ensemble. To obtain this vector, for each $h^s \in \mathcal{H}$ we take all its occurrences in the ensemble, sum the corresponding vectors of parameters, and concatenate the vectors obtained for all $s \in S$.
If $\mathcal{H}$ consists of decision trees, we sum all trees with the same leaves $R_j$ and $\phi(x)$ maps the feature vector to the binary vector encoding all possible leaves containing $x$.
The parameters of the model $\widehat{F}_\tau$ obtained at iteration $\tau$ are denoted by $\widehat{\Theta}_{\tau}$.\footnote{We use notation $\widehat{\Theta}_{\tau}$, $\widehat{F}_{\tau}$ to stress that a process is discrete, while $F(t)$ is used for a continuous process.}

\subsection{Stochastic Gradient Boosting}\label{sec:sgb}

This section describes a classic stochastic gradient boosting (SGB) algorithm~\cite{friedman2002stochastic} using notation convenient for our analysis.
SGB is a recursive procedure that can be characterized by a tuple $\mathcal{B} := \big(\mathcal{H},p\big)$:
$\mathcal{H}$ is a set of weak learners; 
$p(s|g)$ is a probability distribution over weak learners' indices $s \in S$ given a vector $g \in \mathbb{R}^N$ that is typically picked as a vector of gradients. In case of decision trees, we refer to $p(s|g)$ as a distribution over tree structures.

SGB constructs an ensemble of decision tree iteratively. At each step, we compute unbiased gradient estimates $\widehat{g}_\tau \in \mathbb{R}^N$ such that $\mathbb{E}\widehat{g}_\tau = \big(\frac{\partial}{\partial f}L(f_{\widehat{\Theta}_\tau}(x_i), y_i)\big)_{i=1}^N$, where $f_{\widehat{\Theta}_\tau}$ is the currently built model.
Then, we pick a weak learner's index $s_\tau$ according to the distribution $p(s|\widehat{g}_\tau)$.
It remains to estimate the parameters $\theta^{s_\tau}_*$ of this weak learner, which is done as follows:
\begin{align}\label{eq:gb-wl-estimation}
    &\text{minimize \,\,\, }\|\theta^{s_\tau}\|_2^2\,\,\,\,\, \text{ subject to} \nonumber \\
    &\theta^{s_\tau}\in \argmin_{\theta\in\mathbb{R}^{m_{s_\tau}}} \| -\epsilon\widehat{g}_{\tau} - H_{s_\tau}\theta \|_2^2\,,
\end{align}
where $\epsilon > 0$ is a learning rate. In other words, if $\argmin$ returns not a single value but a set, then we choose the value that minimizes $\|\theta^{s_\tau}\|_2^2$.

\begin{algorithm}[t]
   \caption{SGB}
   \label{alg:sgb}
\begin{algorithmic}
   \STATE {\bfseries input:} dataset $\mathcal{D}_N$, learning rate $\epsilon > 0$
   \STATE initialize $\tau = 0$, $f_{\widehat{\Theta}_{0}}(\cdot) = 0$
   \REPEAT
   \STATE estimate gradients $\widehat{g}_\tau\in \mathbb{R}^N$ on $\mathcal{D}_N$ using $f_{\widehat{\Theta}_{\tau}}(\cdot)$
   \STATE sample tree structure $s_\tau \sim p(s|\widehat{g}_\tau)$
   \STATE estimate leaf values $\theta^{s_\tau}_* = -\epsilon \Phi_{s_\tau} \widehat{g}_\tau$
   \STATE update ensemble $f_{\widehat{\Theta}_{\tau+1}}(\cdot) = f_{\widehat{\Theta}_{\tau}}(\cdot) + h^{s_\tau}(\cdot,\theta^{s_\tau}_*)$
   \STATE update $\tau = \tau + 1$
   \UNTIL{stopping criteria met}
   \STATE {\bfseries return:} $f_{\widehat{\Theta}_\tau}(\cdot)$
\end{algorithmic}
\end{algorithm}

The above optimization problem can be solved exactly as $\theta^{s_\tau}_* = - \epsilon \Phi_{s_\tau} \widehat{g}_{\tau}$, where $\Phi_{s_\tau} = (H_{s_\tau}^T H_{s_\tau})^{\dagger} H_{s_\tau}^T$.\footnote{$A^{\dagger}$ denotes the pseudo-inverse of the matrix A, see Appendix~\ref{app:A} for more details.} 
This expression is general and holds for any weak learners linearly depending on parameters. However, for decision trees, the matrix $\Phi_{s_\tau}$ has a simple explicit formula. 
For a decision tree
$h^{s_\tau}(x, \theta^{s_\tau}) = \sum_i \theta_i^{s_\tau} \mathbb{1}_{\{x\in R_i\}}$ we have
$$
(\Phi_{s_\tau})_{i,j} = \frac{\mathbb{1}_{\{x_j\in R_i\}}}{|\{k: x_k \in R_i\}|}.
$$
Thus, $\Phi_{s_\tau} \widehat{g}_{\tau}$ corresponds to averaging the gradient estimates in each leaf of the tree.


After obtaining $\theta^{s_\tau}_*$, the algorithm updates the ensemble as 
\begin{equation}\label{eq:update}
f_{\widehat{\Theta}_{\tau+1}}(\cdot) := f_{\widehat{\Theta}_{\tau}}(\cdot) + h^{s_\tau}(\cdot,\theta^{s_\tau}_*)\,.
\end{equation}

Under the convexity of $L(z, y)$ by $z$ and some regularity assumptions on the triplet $\mathcal{B}$~\cite{boulevard}, the ensemble converges to the optimal one with respect to the closure of the set of all possible finite ensembles. Thus, one can construct a converging SGB algorithm for convex losses.
For non-convex losses, SGB cannot guarantee convergence for the same reasons as stochastic gradient descent (SGD)~--- it gives only a first-order stationarity guarantee \cite{first-order-non-convex} that means not only local minima points but also saddles, which prevents effective optimization.
Our paper fills this gap: we build a globally converging gradient boosting algorithm for convex and non-convex optimization with provable (under some assumptions on $\mathcal{H}$) generalization gap bounds. 

\subsection{Stochastic Gradient Langevin Dynamics}\label{sec:SGLD}

Our algorithm combines SGB with stochastic gradient Langevin dynamics (SGLD), which we briefly introduce here.
Assume that we are given a Lipschitz smooth function $U(\theta):\mathbb{R}^m \rightarrow \mathbb{R}$ such that $U(\theta) \rightarrow \infty$ as $\|\theta\|_2\rightarrow\infty$.
Importantly, the function $U(\theta)$ can be non-convex.
SGLD~\cite{GelfandGAA, Welling2011BayesianLV,DBLP:journals/corr/RaginskyRT17,Erdogdu:2018:GNO:3327546.3327636} aims at finding the global minimum of $U(\theta)$ which is in sharp contrast with SGD that can stuck in local minima point. It performs an iterative procedure and updates $\widehat{\theta}_\tau$ as:
\begin{equation}\label{eq:SGLD_update}
    \widehat{\theta}_{\tau + 1} = \widehat{\theta}_\tau - \epsilon P\widehat{\nabla}U(\widehat{\theta}_\tau) + \mathcal{N}(\mathbb{0}_m, 2\epsilon\beta^{-1} P),
\end{equation}
where $\widehat{\nabla}U(\theta)$ is an unbiased stochastic gradient estimate (i.e., $\mathbb{E}\widehat{\nabla}U(\theta) = \nabla U(\theta)$), $\epsilon > 0$ is a learning rate, $\beta > 0$ is an inverse diffusion temperature, and $P \in \mathbb{R}^{m\times m}$ is a symmetric positive definite preconditioner matrix~\citep{Welling2011BayesianLV}.
Informally, one just injects the $\sqrt{\epsilon}$ rescaled Gaussian noise into a standard SGD update to force the algorithm to explore the whole space and eventually find the global optimum. The diffusion temperature controls the level of exploration.\footnote{For the proposed approach, $\beta$ also affects the generalization. In our experiments, we tune this parameter on the validation set.} The preconditioner matrix $P$ affects the convergence rate of $\widehat{\theta}_{\tau}$ while still giving the convergence to the global optimum~\citep{Welling2011BayesianLV}. We write~\eqref{eq:SGLD_update} in such a general form since in our analysis of SGLB we obtain a particular non-trivial preconditioner matrix.

Under mild assumptions\footnote{Typical assumptions are Lipschitz smoothness and convexity outside a large enough ball, often referred to as dissipativity.} \cite{DBLP:journals/corr/RaginskyRT17,Erdogdu:2018:GNO:3327546.3327636}, the chain $\widehat{\theta}_\tau$ converges in distribution to a random variable with density $p_{\beta}(\theta) \propto \exp(-\beta U(\theta))$ (Gibbs distribution) as $\epsilon\tau \rightarrow \infty$, $\epsilon \rightarrow 0_+$. 
Moreover, according to \citet{DBLP:journals/corr/RaginskyRT17,Erdogdu:2018:GNO:3327546.3327636}, we have \[
\mathbb{E}_{\theta\sim p_\beta(\theta)}U(\theta) - \min_{\theta\in\mathbb{R}^m} U(\theta) = \mathcal{O}\left(\frac{m}{\beta}\log\frac{\beta}{m}\right)\,,
\]
so the distribution $p_\beta(\theta)$ concentrates around the global optimum of $U(\theta)$ for large~$\beta$. 

Following~\citet{DBLP:journals/corr/RaginskyRT17}, we refer to such random variables as \emph{almost minimizers} of the function, meaning that by choosing the parameters of the algorithm, we can obtain $\varepsilon$-minimizers with probability arbitrary close to one for any fixed $\varepsilon > 0$. 

The main idea behind the proof of convergence is to show that the continuous process $\theta_\epsilon(t) := \widehat{\theta}_{[\epsilon^{-1}t]}$ weakly converges to the solution of the following \textit{associated} Langevin dynamics stochastic differential equation (SDE) as $\epsilon \rightarrow 0_+$:
\begin{equation}
    \mathrm{d}\theta(t) = -P\nabla U(\theta(t))\mathrm{d}t + \sqrt{2\beta^{-1}P}\mathrm{d}W(t),
\end{equation}
where $W(t)$ is a standard Wiener process. The Gibbs measure $p_\beta(\theta)$ is an invariant measure of the SDE, $\theta(t)$ is a solution of the SDE typically defined by the Ito integral~\cite{chung1990introduction}.
A key property of the solution $\theta(t)$ is that its distribution converges to the invariant measure $p_\beta(\theta)$ in a suitable distance (see \citet{DBLP:journals/corr/RaginskyRT17} for details).
We use a similar technique to prove the convergence of the proposed SGLB algorithm.

\section{SGLB Algorithm}\label{sec:algorithm}

\begin{algorithm}[t]
\caption{SGLB}
\label{alg:sglb}
\begin{algorithmic}
\STATE {\bfseries input:} dataset $\mathcal{D}_N$, learning rate $\epsilon > 0$, inverse temperature $\beta > 0$, regularization~$\gamma > 0$ 
\STATE initialize $\tau = 0$, $f_{\widehat{\Theta}_{0}}(\cdot) = 0$
\REPEAT
\STATE {estimate} gradients $\widehat{g}_\tau\in \mathbb{R}^N$ on $\mathcal{D}_N$ using $f_{\widehat{\Theta}_{\tau}}(\cdot)$ \\
\STATE {sample} $\zeta, \zeta' \sim \mathcal{N}(\mathbb{0}_N, I_N)$\\ 
\STATE {sample} tree structure $s_\tau \sim p(s|\widehat{g}_\tau+\sqrt{\frac{2N}{\epsilon\beta}}\zeta')$\\
\STATE {estimate} leaf values $\theta^{s_\tau}_* = -\epsilon \Phi_{s_\tau} ( \widehat{g}_\tau+\sqrt{\frac{2N}{\epsilon\beta}}\zeta)$\\
\STATE {update} ensemble $f_{\widehat{\Theta}_{\tau + 1}}(\cdot) $\\ \hfill $ = (1 - \gamma \epsilon)f_{\widehat{\Theta}_{\tau}}(\cdot) + h^{s_\tau}(\cdot,\theta^{s_\tau}_*)$\\
\STATE {update} $\tau = \tau + 1$\\
\UNTIL{stopping criteria met}
\STATE {\bfseries return:} $f_{\widehat{\Theta}_\tau}(\cdot)$
\end{algorithmic}
\end{algorithm}

In this section, we describe the proposed SGLB algorithm that combines SGB with Langevin dynamics.
Surprisingly, a few simple modifications allow us to convert SGB (Algorithm~\ref{alg:sgb}) to a globally converging method (Algorithm~\ref{alg:sglb}). The obtained approach is easy to implement, but proving the convergence is non-trivial (see Sections~\ref{sec:results} and~\ref{sec:proofs}).

We further assume that the loss $L(z, y)$ is Lipschitz continuous and Lipschitz smooth by the variable $z$ and that $\inf_z L(z, y) > -\infty$ for any $y$. Since $\mathcal{L}_N({F})$ is a sum of several $L(\cdot, \cdot)$, it necessarily inherits all these properties.\footnote{We use $\mathcal{L}_N({F})$ and $\mathcal{L}_N({f})$ interchangeably.}

First, we add $L_2$ regularization~\eqref{eq:reg_loss} to the loss for two reasons: 
1) regularization is known to improve generalization,
2) we \textit{do not} assume that $\mathcal{L}_N(F) \rightarrow \infty$ as $\|F\|_2 \rightarrow \infty$, but we need to ensure this property for Langevin dynamics to guarantee the existence of the invariant measure. Lipschitz continuity of the loss implies at most linear growth at infinity, and $L_2$ regularizer~\eqref{eq:reg_loss} grows faster than linearly. Thus, we get $\mathcal{L}_N(F,\gamma) \rightarrow \infty$ as $\|F\|_2 \rightarrow \infty$.

According to the SGB procedure (Section~\ref{sec:sgb}), a new weak learner should fit 
\[-\epsilon\widehat{\nabla}_F \mathcal{L}_N(\widehat{F}_\tau, \gamma) = 
-\epsilon\widehat{\nabla}_F
\left(\mathcal{L}_N(\widehat{F}_\tau) + \frac{\gamma}{2} \|\widehat{F}_\tau\|_2^2\right)\,.
\] 
Note that $-\epsilon\nabla_F\big(\frac{\gamma}{2}\|\widehat{F}_\tau\|_2^2\big) = -\epsilon\gamma \widehat{F}_\tau$, thus we can make the exact step for $\frac{\gamma}{2}\|F\|_2^2$ by shrinking the predictions as $(1 - \gamma \epsilon )\widehat{F}_\tau$. Due to the linearity of the model, it is equivalent to the shrinkage of the parameters: $(1 - \gamma\epsilon)\widehat{\Theta}_\tau$.

Second, we inject Gaussian noise directly into the SGB gradients estimation procedure. Namely, instead of approximating $\widehat{g}_\tau$ by a weak learner~\eqref{eq:gb-wl-estimation}, we approximate $\widehat{g}_\tau + \mathcal{N}(\mathbb{0}_{N}, \frac{2N}{\epsilon\beta} I_{N})$:
\begin{align}
    &\text{minimize \,\,\, }\|\theta^{s_\tau}\|_2^2\,\,\,\,\, \text{ subject to} \nonumber \\
    &\theta^{s_\tau}\in \argmin_{\theta\in\mathbb{R}^{m_{s_\tau}}} \| -\epsilon\left(\widehat{g}_{\tau} + \sqrt{\frac{2 N}{\epsilon\beta}}\zeta\right) - H_{s_\tau}\theta \|_2^2\,,
\end{align}
where $\zeta \sim \mathcal{N}(\mathbb{0}_N, I_N)$. 
Adding Gaussian noise allows us to show that after a proper time interpolation, the process weakly converges to the Langevin dynamics, which in turn leads to the global convergence. 

Finally, we also add noise to the weak learners' selection algorithm. This is done to make the distribution of the trees independent from the gradient noise. Formally, instead of a given distribution $p(s|g)$, we consider the distribution $p_*(s|g) = \int_z p(s|z)\mathcal{N}(z|g, \frac{2N}{\epsilon\beta})\mathrm{d} z$ which is a convolution of $p(s|z)$ with the Gaussian distribution.\footnote{$\mathcal{N}(z|a, \sigma)$ denotes the density of $\mathcal{N}(a, \sigma)$.} We can easily sample from this distribution by first sampling $z \sim \mathcal{N}(z|g, \frac{2N}{\epsilon\beta})$ and then sampling $s \sim p(s|z)$.

The overall SGLB procedure is described in Algorithm~\ref{alg:sglb}.

\section{Convergence of SGLB}\label{sec:results}

In this section, we formulate our main theoretical results. For this, we need some notation. Let
$$
V_{\mathcal{B}} := \Big\{F=\big(f_\Theta(x_i)\big)_{i=1}^N\Big|\forall \text{ ensembles }\Theta\Big\}\subset \mathbb{R}^N\,.
$$
This set encodes all possible predictions of all possible ensembles. 
It is easy to see that $V_{\mathcal{B}}$ is a linear subspace in $\mathbb{R}^N$: the sum of any ensembles is again an ensemble. 

Let us also define
$P_{s_\tau} := H_{s_\tau} \Phi_{s_\tau}$. 
For given unbiased estimates of the gradients $\widehat{g}_\tau \in \mathbb{R}^{N}$ and a weak learner's index $s_\tau \in S$, such operation first estimates $\theta^{s_\tau}_*$ according to~\eqref{eq:gb-wl-estimation} and then converts them to the predictions of the weak learner. Clearly, we have $P_{s_\tau} v \in V_\mathcal{B}$ for any $v\in \mathbb{R}^N$. The following lemma characterizes the structure of $P_{s_\tau}$ and is proven in Appendix~\ref{app:A}.
\begin{lemma}
\label{lemma=projector}
$P_{s_\tau}$ is an orthoprojector on the image of the weak learner $h^{s_\tau}$, i.e., $P_{s_\tau}$ is symmetric, $P_{s_\tau}^2 = P_{s_\tau}$, and $\mathrm{im }P_{s_\tau} = \mathrm{im }H_{s_\tau}$.
\end{lemma}

We are ready to formulate the required properties for the probability distribution $p(s|g)$ that determines the choice of the weak learners:
\begin{itemize}
    \item \textbf{Non-degeneracy}: $\forall s\in S \,\,\exists\, g\in\mathbb{R}^N$ such that for some small $\delta > 0$ we have $p(s|g') > 0$ almost surely for all $g'$ s.t. $\|g' - g\|_2 < \delta$. 
    \item \textbf{Zero-order Positive Homogeneity}: $\forall \lambda > 0\,\,\forall g \in \mathbb{R}^{N}$ we have $p(s|\lambda g) \equiv p(s|g) \,\, \forall s\in S$. 
\end{itemize} 
Informally, the first property requires that for each tree structure (denoted by $s$), we have a non-zero probability, i.e., by feeding different $g \in \mathbb{R}^N$ to $p(s|g)$ and then sampling $s$, we would eventually sample every possible structure. The second property means that the weak learner's selection does not depend on the scale of $g$. This property is usually satisfied by standard gradient boosting algorithms: typically, one sets $p(s|g) = 1$ for the weak learner achieving the lowest MSE error of predicting the vector $g \in \mathbb{R}^N$). Clearly, if we rescale $g$, then the best fit weak learner remains the same as it is a linear function of its parameters.

Recall that we also assume the loss $L(z, y)$ to be Lipschitz continuous and Lipschitz smooth by the variable $z$, and $\inf_z L(z, y) > -\infty\,\forall y$.
Also, from the stochastic gradient estimates $\widehat{g}_\tau$ we require $\|\widehat{g}_\tau - \mathbb{E}\widehat{g}_\tau\|_2 = \mathcal{O}(1)$ with probability one.
Now we can state our main theorem.

\begin{theorem}\label{thm:SDE}
The Markov chain $\widehat{F}_\tau$ generated by SGLB (Algorithm~\ref{alg:sglb}) weakly converges to the solution of the following SDE as $\epsilon \rightarrow 0_+$:
\begin{multline}\label{eq:sde}
    \mathrm{d}  F(t) = - \gamma F(t)\mathrm{d}t - P_{\infty}\nabla_{F}\mathcal{L}_N(F(t))\mathrm{d}t  \\ + \sqrt{2\beta^{-1}P_{\infty}}\mathrm{d}W(t).
\end{multline}
where $P_{\infty} := N\mathbb{E}_{s\sim p_*(s|\mathbb{0}_N)}P_{s}$.
\end{theorem}

Based on this theorem, we prove that in the limit, $\mathcal{L}_N(\widehat{F}_\tau)$ concentrates around the global optimum of~$\mathcal{L}_N({F})$.

\begin{theorem}\label{thm:convergence}
The following bound holds almost surely: $$\lim\mathbb{E}\mathcal{L}_N(\widehat{F}_\tau) -  \inf_{{F}\in V_{\mathcal{B}}}\mathcal{L}_N({F}) = \mathcal{O}\Big(\delta_\Gamma(\gamma) + \frac{d}{\beta}\log\frac{\beta}{d}\Big)$$
for $\epsilon\rightarrow 0_+, \epsilon\tau\rightarrow \infty$, where $d = \dim V_{\mathcal{B}}$ and $\delta_{\Gamma}(\gamma)$ encodes the error from the regularization and is of order $o(1)$ as $\gamma \rightarrow 0_+$.
\end{theorem}

Note that this theorem is stronger than the convergence in probability.
Indeed, the random variable $l_{\tau} = \mathcal{L}_N(\widehat{F}_\tau)$ satisfies $l_{\tau} \ge l_* = \inf_{{F}\in V_{\mathcal{B}}}\mathcal{L}_N({F})$ almost surely and Theorem~\ref{thm:convergence} implies that we can force $\mathbb{E} l_{\tau} \le l_* + \varepsilon$ for arbitrary $\varepsilon > 0$. Thus, by considering a non-negative random variable $l_{\tau} - l_* \ge 0$ and applying Markov's inequality, we get $\mathbb{P}(l_{\tau} - l_* > k \varepsilon) \le 1 - 1/k$
for arbitrary $k > 0$. 
Thus, for any $\delta > 0$ and $\eta > 0$, we can take $k = 1/\delta$ and $\varepsilon = \eta\delta$ and obtain:
$$\mathbb{P}(\mathcal{L}_N(\widehat{F}_\tau) - \inf_{{F}\in V_{\mathcal{B}}}\mathcal{L}_N({F}) > \eta) \le 1 - \delta,$$
which is exactly the convergence to the $\eta$-minimum with probability at least $1 - \delta$, where $\delta$ and $\eta$ can be made arbitrary small.

To derive a generalization bound, we need an additional assumption on the weak learners set. 
We assume the \emph{independence} of prediction vectors, i.e., that $(h^{s_1}(x_i,\theta^{s_1}))_{i=1}^N, \ldots, (h^{s_k}(x_i,\theta^{s_k}))_{i=1}^N$ are linearly independent for an arbitrary choice of different weak learners indices $s_1, \ldots, s_k \in S$ and parameters $\theta^{s_j}\in \mathbb{R}^{m_{s_j}}$ such that the vector $(h^{s_j}(x_i,\theta^{s_j}))_{i=1}^N$ is not zero for any $j$.
This condition is quite restrictive,\footnote{We are currently working on a much less restrictive generalization result for gradient boosting methods.} but it allows us to deduce a generalization gap bound, which gives an insight into how the choice of $\mathcal{B}$ affects the generalization.
The independence assumption \textit{can be} satisfied if we assume that each feature vector $x\in \mathbb{R}^k$ has independent binary components and $N \gg 1$. Also, we may enforce this condition by restricting the space $\mathcal{H}$ at each iteration, taking into account the currently used weak learners. The latter is a promising direction for future analysis.

\begin{theorem}\label{thm:generalization}
If the above assumptions are satisfied, then the generalization gap can be bounded as
\begin{multline*}
\big|\mathbb{E}_{\Theta\sim p_\beta(\Theta)} \mathcal{L}(f_\Theta(\cdot)) - \mathbb{E}_{\Theta\sim p_\beta(\Theta)}\mathcal{L}_N(f_\Theta(\cdot))\big| \\ = \mathcal{O}\Big(\frac{(\beta+d)^2}{\lambda_* N}\Big)\,,
\end{multline*}
where $p_\beta(\Theta)$ is the limiting distribution of $\Theta_t$ as $\epsilon \rightarrow 0_+, \epsilon t\rightarrow \infty$;
$d = \dim V_{\mathcal{B}}$ is independent from $N$ for large enough $N \gg 1$; and $\lambda_*$ is a so-called uniform spectral gap parameter that encodes ``hardness'' of the problem and depends on the choice of $L(z, y)$.
\end{theorem}

Note that $1/\lambda_*$ is not dimension-free and in general may depend on $d$, $\beta$, $\Gamma$, and $\gamma$. Moreover, the dependence can be exponential in $d$~\cite{DBLP:journals/corr/RaginskyRT17}. However, such exponential dependence is very conservative since the bound applies to \emph{any} Lipschitz continuous and Lipschitz smooth function. For instance, in the convex case, we can get a dimension-free bound for $1/\lambda_*$. We refer to Section~\ref{sec:proof_generalization} for further details and intuitions behind the statement of Theorem~\ref{thm:generalization}.

\begin{corollary}\label{corollary}
It follows from Theorems~\ref{thm:convergence} and~\ref{thm:generalization} that 
SGLB has the following performance as $\epsilon \rightarrow 0_+$, $\epsilon \tau \rightarrow \infty$:
\begin{multline}\label{eq:final_bound}
    \big|\lim \mathcal{L}(\widehat{F}_\tau) - \inf_{F \in V_{\mathcal{B}}} \mathcal{L}(F)\big| \\ = \mathcal{O}\Big(\delta_\Gamma(\gamma) + \frac{d}{\beta}\log\frac{\beta}{d} + \frac{(\beta+d)^2}{\lambda_* N}\Big).
\end{multline}
\end{corollary}
Here $\delta_\Gamma(\gamma)$ encodes the error from the regularization that is negligible in $\gamma \rightarrow 0_+$ limit.
However, $1/\lambda_*$ also depends on $\gamma$ (e.g., in the convex case the dependence is of order $\mathcal{O}(1/\gamma)$) and thus the optimal $\gamma_*$ must be strictly greater than zero.
Finally, taking large enough $N$, the bound~\eqref{eq:final_bound} can be made arbitrarily small, and therefore our algorithm reaches the ultimate goal stated in Section~\ref{sec:sgb}.
The obtained bound answers  the question how the choice of the tuple $\mathcal{B} = (\mathcal{H}, p)$ affects the optimization quality. Note that our analysis, unfortunately, gives no insights on the speed of the algorithm's convergence and we are currently working on filling this gap.

\section{Proofs}\label{sec:proofs}

\subsection{Preliminaries}

Before we prove the main theorems, let us analyze the distribution $p_*(s|g)$.

\begin{lemma}\label{lemma:lipshitz-p}
$p_*(s|g) = p_*(s|\mathbb{0}_N) + \mathcal{O}(\|g\|_2)$ as $\|g\|_2 \rightarrow 0$.
\end{lemma}
\begin{proof}
The statement comes from the uniform boundedness of probabilities $p(s|g)$. By definition, $p_*(s|g)$ is a Gaussian smoothing of $p(s|g)$, and by properties of the Gaussian smoothing it is a Lipschitz-continuous function~\cite{Nesterov2017RandomGM}.
\end{proof}
\begin{statement}
$\mathbb{E}_{s\sim p_*(s|g)} P_s$ is non-degenerate as an operator from $V_{\mathcal{B}}$ to $V_{\mathcal{B}}$ $\forall g\in\mathbb{R}^N$.
\end{statement}
\begin{proof}
First, we use the Non-degeneracy property of $p(s|g)$: it implies that after smoothing $p(s|g)$ into $p_*(s|g)$ we have $p_*(s|g) > 0\,\,\forall s\in S \,\,\forall g\in\mathbb{R}^N$. Finally, we note that $V_{\mathcal{B}}$ is formed as a span of both images and coimages of $P_s$. Therefore, $\mathbb{E}_{s\sim p_*(s|g)} P_s$ is a non-degenerate operator within $V_{\mathcal{B}}$. 
\end{proof}

\subsection{Proof of Theorem~\ref{thm:SDE}}\label{sec:proof_1}
We heavily rely on Theorem~1 (p.~464) of \citet{gikhman1996introduction}. We note that this theorem is proved in dimension one in \citet{gikhman1996introduction}, but it remains valid in an arbitrary dimension~\cite{kushner1974}. We formulate a slightly weaker version of the theorem that is sufficient for our case and follows straightforwardly from the original one.
\begin{theorem}[\citet{gikhman1996introduction}]\label{thm:skorokhod} Assume that we are given a Markov chain $\widehat{F}_\tau$ that satisfies the following relation:
\begin{align*}
\epsilon^{-1}\mathbb{E}\big(\widehat{F}_{\tau+1}-\widehat{F}_\tau\big|\widehat{F}_\tau\big) &= A(\widehat{F}_\tau) + \mathcal{O}(\sqrt{\epsilon})\,, \\
\epsilon^{-1}\mathrm{Var}\big(\widehat{F}_{\tau+1}\big|\widehat{F}_\tau\big) &= B(\widehat{F}_\tau)(B(\widehat{F}_\tau))^T + \mathcal{O}(\sqrt{\epsilon})
\end{align*}
for some vector and matrix fields $A(\cdot), B(\cdot)$ that are both Lipschitz and $B(\cdot)$ is everywhere non-degenerate matrix for every argument. Then, for any fixed $T > 0$, we have that $F_{\epsilon}(t) := \widehat{F}_{\left[\epsilon^{-1}t\right]}$ converges \emph{weakly} to the solution process of the following SDE on the interval $t \in [0, T]$:
$$\mathrm{d} F(t) = A(F(t))\mathrm{d}t + B(F(t))\mathrm{d}W(t),$$
where $W(t)$ is a standard Wiener process.

\end{theorem}

Recall that $P_{s_\tau}$ is an orthoprojector defined by the sampled weak learner's index $s_\tau$. Let us denote the corresponding \emph{random} projector by $P_{\tau}$ ($P_{\tau}$ depends on a randomly sampled index).
Then, $\widehat{F}_\tau$ conforms the following update equation:
\begin{equation*}
\widehat{F}_{\tau + 1} = (1 - \gamma\epsilon)\widehat{F}_\tau - \epsilon N P_\tau \widehat{\nabla}_{F}\mathcal{L}_N(\widehat{F}_\tau) + \sqrt{2\epsilon\beta^{-1}N P_\tau}\zeta\,.
\end{equation*}
Here we exploit the fact that $P_\tau$ is a projector: $P_\tau^2 \equiv P_\tau$. The equation clearly mimics the SGLD update with the only difference that $P_\tau$ is not a constant preconditioner but a random projector.
So, the SGLB update can be seen as SGLD on random subspaces in $V_\mathcal{B}$. 

Recall that we require $\|\widehat{g}_\tau - \nabla_F \mathcal{L}_N(\widehat{F}_\tau)\|_2 \le C_1$ with probability one for some constant $C_1$.
Also, $\nabla_F \mathcal{L}_N(F)$ is uniformly bounded by some constant $C_2$ due to Lipschitz continuity. So,
$\widehat{g}_\tau$ is uniformly bounded with probability one by a constant $C_3 = C_1 + C_2$. 

Using Zero-order positive homogeneity, we get $p_*(s|g) = p_\infty(s) + \mathcal{O}(\sqrt{\epsilon})$, where $ p_\infty(s)  := p_*(s|\mathbb{0}_N)$. To prove this, we note that 
\begin{multline*}
p_*(s|g) = \mathbb{E}_{\varepsilon\sim\mathcal{N}(\mathbb{0}_N, I_N)} p\left(s\big|g + \sqrt{\frac{2N}{\epsilon\beta}}\varepsilon\right) \\ = \mathbb{E}_{\varepsilon\sim\mathcal{N}(\mathbb{0}_N, I_N)} p\left(s\big|\sqrt{\epsilon}\sqrt{\frac{\beta}{2N}} g + \varepsilon\right) \\ = p_{\infty}(s) + \mathcal{O}\left(\sqrt{\epsilon}C_3\sqrt{\frac{\beta}{2N}}\right)\,,
\end{multline*}
where the final equality comes from Lemma \ref{lemma:lipshitz-p}. So, noting that $C_4 := C_3\sqrt{\frac{\beta}{2}}$ is a constant, we obtain that the distribution $p_*(s|g)$ converges to $p_\infty(s)$ as $\epsilon\rightarrow 0_+$ uniformly for each $s$ (note that $s\in S$ and we assumed $|S| < \infty$). 

Let us define an \textit{implicit} limiting preconditioner matrix of the boosting algorithm $P_{\infty} := N\mathbb{E}_{s\sim p_\infty(s)}H_{s}\Phi_s:\mathbb{R}^{N}\rightarrow\mathbb{R}^{N}$. This expectation exists since each term is a projector and hence is uniformly bounded by 1 using the spectral norm. Since $p_*(s|g) = p_\infty(s) + \mathcal{O}(\sqrt{\epsilon N^{-1}})$, we get $N\mathbb{E}_{s\sim p_*(s|g)}H_{s}\Phi_s = P_{\infty} + \mathcal{O}(\sqrt{\epsilon N})$ in operator norm since $P_s = H_{s}\Phi_s$ is essentially bounded by one in the operator norm. Therefore, we obtain
\begin{align*}
\epsilon^{-1}\mathbb{E}(\widehat{F}_{\tau+1}{-}\widehat{F}_\tau|\widehat{F}_\tau) &= -\gamma \widehat{F}_\tau{-}P_{\infty}\nabla_{F}\mathcal{L}_N(\widehat{F}_\tau)+ \mathcal{O}(\sqrt{\epsilon}), \\
\epsilon^{-1}\mathrm{Var}(\widehat{F}_{\tau+1}|\widehat{F}_\tau) &= 2\beta^{-1} P_{\infty} + \mathcal{O}(\sqrt{\epsilon}).
\end{align*}
Henceforth, Theorem \ref{thm:skorokhod} applies ensuring the weak convergence of ${F}_{\epsilon}(t) := \widehat{F}_{\big[\epsilon^{-1}t\big]}$ to the Langevin Dynamics as $\epsilon\rightarrow 0_+$ for $t \in [0, T]$ $ \,\,\forall T > 0$. 

\subsection{Proof of Theorem~\ref{thm:convergence}}\label{sec:proof_2}

To prove the theorem, we study the properties of the limiting Langevin equation, which describes the evolution of $F(t)$ in the space $V_\mathcal{B}$. The trick is to observe that $V_{\mathcal{B}} = \mathrm{im }P_{\infty} = \mathrm{coim }P_{\infty}$ due to Non-degeneracy of the sampling $p(s|g)$. Henceforth, we can easily factorize $V_{\mathcal{B}}\oplus \ker P_{\infty} = \mathbb{R}^N$ and assume that actually we live in $V_{\mathcal{B}}$ and \textit{there} formally $P_{\infty} > 0$ as an operator from $V_{\mathcal{B}}$ to~$V_{\mathcal{B}}$.\footnote{$P > 0$ means that all eigenvalues of $P$ are positive reals.}

Now, consider the following Langevin-like SDE for $P > 0, P^T = P$ in $\mathbb{R}^d$: 
\begin{equation}\label{eq:sde-1}
\mathrm{d}F(t) = -\gamma F(t)\mathrm{d}t  - P\nabla_{F}\mathcal{L}_N(F(t) )\mathrm{d}{t} + \sqrt{2\beta^{-1}P}\mathrm{d}W\,.
\end{equation}
The classic form of the equation can be obtained using an implicitly regularized loss function:\footnote{Here we redefine the notation $\mathcal{L}_N({F}, \gamma)$ used before.}
$$
\mathcal{L}_N({F}, \gamma) := \mathcal{L}_N({F}) + \frac{\gamma}{2}\|\Gamma{F}\|_2^2\,,
$$ 
where $\Gamma := \sqrt{P^{-1}} > 0$ is a regularization matrix. Due to the well-known properties of $L_2$-regularization, for small enough $\gamma > 0$, the minimization of $\mathcal{L}_N({F}, \gamma)$ leads to an almost minimization of $\mathcal{L}_N({F})$ with an error of order $\mathcal{O}(\delta_\Gamma(\gamma))$ for some function $\delta_\Gamma(\gamma)$ depending on $\mathcal{L}_N$ and $\Gamma$, which is negligible as $\gamma \rightarrow 0_+$.\footnote{If $\min_F \mathcal{L}_N(F)$ exists, one can show that $\delta_\Gamma(\gamma) = \mathcal{O}(\lambda_{\max}(\Gamma^2) \gamma)$ with $\mathcal{O}(\cdot)$ depending on $\mathcal{L}_N$.}
Thus, the error $\delta_\Gamma(\gamma)$ heavily depends on $\Gamma = \sqrt{P^{-1}}$.

Using $\mathcal{L}_N({F}, \gamma)$, we rewrite Equation~\eqref{eq:sde-1} as:
$$
\mathrm{d}F(t) = -P\nabla_{F}\mathcal{L}_N(F(t), \gamma)\mathrm{d}t + \sqrt{2\beta^{-1}P}\mathrm{d}W(t).
$$
Then, the results of \citet{Erdogdu:2018:GNO:3327546.3327636, DBLP:journals/corr/RaginskyRT17} apply ensuring that $\mathcal{L}_N(\widehat{F}_\tau, \gamma)$ converges to an almost-minimizer of $\mathcal{L}_N({F}, \gamma)$ with an error of order $\mathcal{O}\Big(\frac{d}{\beta}\log\frac{\beta}{d}\Big)$, where $d = \dim V_{\mathcal{B}}$. From this the theorem follows.


\subsection{Proof of Theorem~\ref{thm:generalization}}\label{sec:proof_generalization}

We reduced the problem of convergence of a general boosting to the problem of convergence of \textit{predictions} $\widehat{F}_\tau := \big(f_{\widehat{\Theta}_\tau}(x_i)\big)_{i=1}^N$ in the space of all possible ensemble predictions $V_\mathcal{B}$ on the dataset $\mathcal{D}_N$. 

Using that $|S| < \infty$, we define a \textit{design} matrix $\Psi := \big[\phi(x_1),\ldots, \phi(x_N)\big]^T \in \mathbb{R}^{N\times m}$, so we can write ${F} = \Psi \Theta$ and $\mathcal{L}_N({F}, \gamma) = \mathcal{L}_N(\Psi\Theta) + \frac{\gamma}{2}\|\Gamma \Psi \Theta\|_2^2$. Note that $V_\mathcal{B}$ can be obtained as the image of $\Psi$. 

Let us consider the uniform spectral gap parameter\footnote{We refer to~\citet{DBLP:journals/corr/RaginskyRT17} for the definition of a uniform spectral gap.} $\lambda_* \ge 0$ for the distribution
\[p_\beta(\Theta) := \frac{\exp(-\beta \mathcal{L}_N(\Psi\Theta, \gamma))}{\int_{\mathbb{R}^m}\small{\exp(-\beta \mathcal{L}_N(\Psi\Theta, \gamma))}\mathrm{d}\Theta}\,.\]
To show that $p_\beta(\Theta)$ is a completely and correctly defined distribution, we first note that $\ker \Psi = \bigoplus_{s\in S} \ker h^{s}$, where $\ker h^s := \{\theta^s \in \mathbb{R}^{m_s}: h(x_i,\theta^s) = 0\,\,\forall x_i\in \mathcal{D}_N \}\subset \mathbb{R}^{m_s}$. Indeed, this is equivalent to our assumption on the independence of prediction vectors. Then, $\ker \Psi$ has right structure, i.e., we have ``basis weak learners'' for $\mathcal{B}$.

Second, we use the factorization trick to factorize $\mathbb{R}^{m_s} = \ker h^s \oplus (\ker h^s)^{\perp} $ and hence w.l.o.g.~we can assume $\ker h^s = \{\mathbb{0}_{m_s}\}$. The latter implies that w.l.o.g.~$\ker \Psi = \{\mathbb{0}_m\}$, so the distribution $p_\beta(\Theta)$ is a correctly defined distribution on $\mathbb{R}^m$. Moreover, observe that in that case $m = d = \dim V_{\mathcal{B}}$ and thus $d$ is \textit{independent} from $N$ for large enough $N \gg 1$.

Having $p_\beta(\Theta) \propto \exp(-\beta \mathcal{L}_N(\Psi\Theta, \gamma))$, from \citet{DBLP:journals/corr/RaginskyRT17}, we can transfer a bound $\propto \frac{(\beta+d)^2}{\lambda_* N}$ for the generalization gap. 
Since we added $L_2$-regularizer to the loss, which is Lipschitz smooth and continuous, we necessarily obtain dissipativity and thus $\lambda_* > 0$~\cite{DBLP:journals/corr/RaginskyRT17}.
This concludes the proof of the theorem.

In the presence of convexity, we can get a dimension-free bound for $1/\lambda_*$.  To see that, we need to bound $1/\lambda_* \le c_{P}(p_\beta)$, where $c_{P}$ is the Poincare constant for $p_\beta(\Theta)$. If $L(\cdot, \cdot)$ is convex, then $\beta\mathcal{L}_N(\Theta, \gamma)$ must be strongly convex with a constant $\ge \frac{\kappa \beta\gamma}{2}$, where $\kappa := \lambda_{\min}(\Psi\Gamma^2 \Psi^T)$ is the smallest eigenvalue of $\Psi\Gamma^2 \Psi^T > 0$. Then, $p_{\beta}$ is strongly log-concave, so by transferring the Poincare inequality for strongly log-concave distribution from \citet{Milman2007arXiv0712.4092M}, we obtain a dimension free-bound $1/\lambda_*\le \frac{1}{\kappa\gamma\beta}$.

\section{Experiments}\label{sec:experiments}

\subsection{Implementation}

Our implementation of SGLB is available within the open-source CatBoost gradient boosting library. CatBoost is known to achieve state-of-the-art results across a wide variety of datasets~\cite{catboost}. As the baseline, we use the standard SGB algorithm implemented within the same library.

Recall that in Section~\ref{sec:results} we require several properties to be satisfied for the distribution $p(s|g)$. Importantly, they are all satisfied for CatBoost, as we show in Appendix~\ref{sec:catboost}.

\subsection{Direct Accuracy Optimization via Smooth Loss Approximation}
\label{sec:direct01}

To show the power of SGLB for non-convex multimodal optimization, we select $\text{Accuracy}$ (0-1 loss) for the direct optimization by our framework:
\begin{multline}\label{eq:0-1}
\text{0-1 loss} = 1 - \text{accuracy} \\ = 1 - \frac{1}{N}\sum_{i=1}^N \mathbb{1}_{\{(2y_i-1) f(x_i|\theta) > 0\}},
\end{multline}
where $y_i \in \{0, 1\}$.
To make this function Lipschitz smooth and Lipschitz continuous, we approximate the indicator by a sigmoid function and minimize:
\begin{equation}
    \mathcal{L}_N(f) := 1-\frac{1}{N}\sum_{i=1}^N \sigma(\varsigma^{-1} (2y_i-1) f(x_i|\theta)), 
\end{equation}
where $\sigma(x) = (1+\exp(x))^{-1}$ and $\varsigma > 0$ is a hyperparameter.
Such smoothing is known as Smooth Loss Approximation (SLA)~\cite{nguyen2013algorithms}. If $f(x_i|\theta) \ne 0\,\,\forall i$, then $\mathcal{L}_N(f)$ converges to 0-1 loss as $\varsigma \rightarrow 0_+$.

To apply SGLB, we need to ensure Lipschitz smoothness and continuity. Observe that the gradient is uniformly bounded due to $\big|\frac{\mathrm{d}}{\mathrm{d}z}\sigma(z)\big| = (1 - \sigma(z))\sigma(z) \le 1$, which in turn implies Lipschitz continuity. Moreover, $\big|\frac{\mathrm{d}^2}{\mathrm{d}z^2}\sigma(z)\big| = \big|(1 - 2\sigma(z))\cdot\frac{\mathrm{d}}{\mathrm{d}z}\sigma(z)\big| \le 3$ which implies Lipschitz smoothness. Thus, SGLB is applicable to SLA.

\subsection{Illustration on Synthetic Data}

First, we analyze the performance of SGLB in a simple synthetic experiment. We randomly generate three-dimensional feature vectors $x \sim \mathcal{N}(\mathbb{0}_3, I_3)$ and let $y = \mathbb{1}_{\{\widetilde{y} > 0\}}$ for $\widetilde{y} \sim \mathcal{N}(\sin(x_1 x_2 x_3), 1)$. We add a significant amount of noise to the target since in this case the loss is very likely to be multimodal.

We made cross-validation with 100 folds, each containing 1000 examples for training and 1000 examples for testing.
To see the difference between the methods, we consider simple models based on decision trees of depth 1. 
We fix $border\_count=5$ (the number of different splits allowed for each feature). We set learning rate to 0.1 and $\varsigma = 10^{-1}$ for SLA. For SGLB, we set $\beta = 10^3$ and $\gamma = 10^{-3}$. Moreover, we set the subsampling rate of SGB to 0.5. 

\begin{table}
\caption{Optimization on synthetic data}
\label{tab:0-1_loss-synt}
\vspace{5pt}
\centering
\begin{tabular}{lcc}
\toprule
Approach & 0-1 loss & p-value \\
\midrule
Logloss + GB & 0.482 & $2\cdot 10^{-8}$\\
SLA + GB & 0.475 & $5\cdot 10^{-3}$ \\
SLA + SGB & 0.474 & $7 \cdot 10^{-3}$\\
SLA + SGLB & \bf 0.470 & --- \\
\bottomrule
\end{tabular}
\end{table}
The results are presented in Table~\ref{tab:0-1_loss-synt}. Here the p-values (according to the $t$-test) are reported relative to SGLB. We see that for the SLA optimization, SGLB outperforms the classic gradient boosting (GB). Then, we analyze whether a standard subsampling used in SGB can help to avoid local optima and improve the performance of GB. We see that SGB with a sample rate of 0.5 is indeed slightly better than GB but is still worse than SGLB, which has a theoretically grounded randomization. Finally, we note that optimizing the convex logistic loss instead of SLA leads to the worst performance. This means that the logistic loss has a different optimum and, for better performance, non-convex loss functions should not be replaced by convex substitutes. 

\subsection{Comparison on Real Data}

In this section, we show that SGLB has superior performance on various real datasets. 
We mostly focus on classification task with 0-1 loss~\eqref{eq:0-1} 
and optimize SLA $L(z, y) = 1 - \sigma(\varsigma^{-1}z(2y-1))$ with $\varsigma = 10^{-1}$, which is Lipschitz smooth and continuous and thus SGLB can be applied.

The datasets are described in Table~\ref{tab:datasets} in Appendix.
We split each dataset into train, validation, and test sets in proportion 65/15/20. We tune the parameters on the validation set using 200 iterations of random search and select the best iteration on the validation; the details are given in Appendix~\ref{app:C2}. For all algorithms, the maximal number of trees is set to 1000. 

The results are shown in Table~\ref{tab:0-1_loss}. We use bold font to highlight significance for the two-tailed $t$-test with a p-value $< 0.05$. We note that SGB uses leaves regularization, while SGLB is not.
We see that for the \textit{non-convex} 0-1 loss, SGLB performance is \textit{superior} to SGB, which clearly shows the necessity of non-convex optimization methods in machine learning.

We also compare SGB with SGLB for the \emph{convex} Logistic regression loss $L(z, y) = -y \log \sigma(z) - (1 - y) \log (1 - \sigma(z))$.
Similarly to accuracy, it is easy to show that this loss is Lipschitz smooth and Lipschitz continuous as $\frac{\mathrm{d}}{\mathrm{d}z}L(z, y) = -y + \sigma(z)$ and $\frac{\mathrm{d^2}}{\mathrm{d}z^2}L(z, y) = \sigma(z) (1 - \sigma(z))$.
The results of the comparison are shown in Table~\ref{tab:logloss}. We see that in most cases, SGLB and SGB are comparable, but SGLB is preferable. Importantly, SGLB has better performance on large Epsilon and Higgs datasets.

\begin{table}
\caption{0-1 loss optimized via SLA}
\label{tab:0-1_loss}
\vspace{5pt}
\centering
\begin{tabular}{lccc}
\toprule
Dataset   & SGB  & SGLB     & p-value \\
\midrule
Appetency & 1.8  & 1.8      & 1 \\
Churn     & 7.1  & 7.2      & 0.18\\
Upselling & 4.8  & \bf 4.7  &  0.04 \\
Adult     & 13.2 & \bf 12.8 & 0.01\\
Amazon    & 5.2  & \bf 4.8  & 0.01\\
Click     & 16   & \bf 15.9 & $3 \cdot  10^{-6}$\\
Epsilon   & 11.7 & 11.7     & 0.44\\
Higgs     & 25.2 & \bf 24.8 & 0.04\\ 
Internet  & 10.1 & \bf 9.8  & 0.05 \\
Kick      & 9.7  & \bf 9.6  & 0.02 \\
\bottomrule
\end{tabular}
\end{table}

\section{Conclusion \& Future Work}\label{sec:conclusion}

Our experiments demonstrate that the theoretically grounded SGLB algorithm also shows promising experimental results. Namely, SGLB can be successfully applied to the optimization of 0-1 loss that is known to be non-convex. Interestingly, since the first version of this paper appeared online, SGLB has found other exciting applications. For instance,~\citet{ustimenko2020stochasticrank} show that SGLB can be applied to learning to rank, which is a classic information retrieval problem. The authors propose to smooth any given ranking loss and then directly optimize it. The obtained function turns out to be non-convex, and SGLB provably allows for reaching the global optimum. \citet{malinin2021uncertainty} apply SGLB to uncertainty estimation. Namely, they use the convergence of parameters to the stationary distribution $p_\beta(\Theta)$ to sample from the posterior, which allows for theoretically grounded uncertainty estimates.

There are plenty of directions for future research that can potentially further improve the performance of SGLB. Recall that our generalization gap estimate relies on the restrictive assumption on linear independence of weak learners. Thus, a promising direction is to modify the algorithm so that some form of Langevin diffusion is still preserved in the limit with good provable generalization gap guarantees. Another idea is to incorporate momentum into boosting so that there is the Hamiltonian dynamics~\citep{gao2018global} in the limit instead of the ordinary Langevin dynamics. There are several theoretical attempts to incorporate momentum into boosting like HistoricalGBM~\citep{GCAI-2018:Historical_Gradient_Boosting_Machine}, so the question is: if we use the HistoricalGMB approach or its modification, would that be enough to claim the Hamiltonian dynamics equation in the limit? Finally, our research does not investigate the rates of convergence, which is another promising direction. It would give a better understanding of the trade-offs between the algorithm's parameters.

\begin{table}
\caption{Logloss optimization}
\label{tab:logloss}
\vspace{5pt}
\centering
\begin{tabular}{lccc}
\toprule
Dataset   & SGB       & SGLB      & p-value \\
\midrule
Appetency & 0.074     & 0.073     & 0.6 \\
Churn     & \bf 0.229 & 0.230     & 0.02\\
Upselling & 0.164     & \bf 0.163 & 0.05\\
Adult     & 0.274     & 0.276     & 0.07\\
Amazon    & 0.144     & 0.145     & 0.09\\
Click     & 0.395     & \bf 0.394 & $4 \cdot 10^{-4}$\\
Epsilon   & 0.274     & \bf 0.273 & 0.01 \\
Higgs     & 0.480     & \bf 0.479 & $2\cdot 10^{-35}$ \\
Internet  & 0.226     & 0.225     & 0.26\\
Kick      & \bf 0.288 & 0.289     & 0.04\\
\bottomrule
\end{tabular}
\end{table}

\bibliography{boosting}
\bibliographystyle{icml2021}


\appendix

\section{Proof of Lemma~\ref{lemma=projector}}\label{app:A}

First, let us prove that $\Phi_{s_\tau} = (H_{s_\tau}^T H_{s_\tau})^{\dagger} H_{s_\tau}^T$.

We can rewrite 
Equation~\eqref{eq:gb-wl-estimation} as 
$$
\theta_*^{s_\tau} = \lim_{\delta \rightarrow 0} \argmin_{\theta^{s_\tau}}\|-\epsilon\widehat{g}_{\tau} - H_{s_\tau}\theta^{s_\tau}\|_2^2 \\ + \delta^2 
\|\theta^{s_\tau} \|_2^2 \,.
$$ 
Taking the derivative of the inner expression, we obtain:
\begin{equation*}
    \left(H_{s_\tau}^T H_{s_\tau} + \delta^2 I_N\right) \theta^{s_\tau} - \epsilon H_{s_\tau}^T \widehat{g}_{\tau} = 0\,.
\end{equation*}
So, $\Phi_{s_\tau}$ can be defined as $\lim_{\delta\rightarrow 0} (H^T_{s_\tau} H_{s_\tau} + \delta^2 I_N)^{-1}H_{s_\tau}^T$. Such limit is well defined and is known as the pseudo-inverse of the matrix~\citep{Gulliksson2000}. 

Let us now prove Lemma~\ref{lemma=projector}.

The matrix $P_{s_\tau}$ is symmetric since $P_{s_\tau} = \lim_{\delta\rightarrow 0} H_{s_\tau}(H^T_{s_\tau} H_{s_\tau} + \delta^2 I_N)^{-1}H_{s_\tau}^T$.

Observe that if $H_{s_\tau}\theta^{s_\tau} = v$, then $P_{s_\tau} v = v$, since the problem
in Equation~\eqref{eq:gb-wl-estimation} 
has an exact solution for the $\argmin$ subproblem. As a result, $\mathrm{im} P_{s_\tau} = \mathrm{im }H_{s_\tau}$. Also, for an arbitrary $v\in \mathbb{R}^N$, we have $P_{s_\tau} (P_{s_\tau} v) = P_{s_\tau} v$ since $P_{s_\tau} v \in \mathrm{im } H_{s_\tau}$.

\section{CatBoost Implementation}\label{sec:catboost}

We implemented SGLB as a part of the CatBoost gradient boosting library, which was shown to provide state-of-the-art results on many datasets~\citep{catboost}.
Now we specify the particular tuple $\mathcal{B}= (\mathcal{H}, p(s|g))$ such that all the required assumption are satisfied. Therefore, the implementation must converge globally for a wide range of functions, not only for convex ones. 

Let us describe the weak learners set $\mathcal{H}$ used by CatBoost.
For each numerical feature, CatBoost chooses between a finite number of splits $\mathbb{1}_{\{x_i \le c_{ij}\}}$, where $\{c_{ij}\}_{j=1}^{d_i}$ are some constants typically picked as quantiles of $x_i$ estimated on $\mathcal{D}_N$ and $d_i$ is bounded by a hyperparameter $\textit{border-count}$.
So, the set of weak learners $\mathcal{H}$ consists of all non-trivial binary oblivious trees with splits $\mathbb{1}_{\{x_i \le c_{ij}\}}$ and with depth bounded by a hyperparameter $\textit{depth}$. 
This set is finite, $|S| < \infty$.
We take $\theta^s \in \mathbb{R}^{m_s}$ as a vector of leaf values of the obtained tree. 

Now we are going to describe $p(s|g)$. Assume that we are given a vector $g \in \mathbb{R}^N$ and already built a tree up to a depth $j$ with remaining (not used) binary candidate splits $b_1, \ldots b_p$. Each split, being added to the currently built tree,  divides the vector $g$ into components $g_1\in \mathbb{R}^{N_1}, \ldots, g_{k} \in \mathbb{R}^{N_k}$, where $k = 2^{j+1}$. To decide which split $b_l$ to apply, CatBoost calculates the following statistics:
$$s_l := \sqrt{\sum_{i=1}^k\text{Var}(g_i) },$$
where $\text{Var}(\cdot)$ is the variance of components from the component-wise mean. Denote also $\sigma := \sqrt{\text{Var}(g)}$. Then, CatBoost evaluates:
$$s_l' := \mathcal{N}\left(s_l,\left(\frac{\rho \sigma}{1 + N^{\epsilon \tau}}\right)^2\right),$$
where $\rho \ge 0$ is a hyperparameter defined by the $\textit{random-strength}$ parameter. After obtaining $s_l'$, CatBoost selects the split with a highest $s_l'$ value and adds it to the tree. Then, it proceeds recursively until a stopping criteria is met.

Since $\epsilon\tau \rightarrow \infty$, we can assume that the variance of $s_l'$ equals zero in the limit. Thus, the stationarity of sampling is preserved. So, $p(s|g)$ is fully specified, and one can show that it satisfies all the requirements. Henceforth, such CatBoost implementation $\mathcal{B}$ must converge globally for a large class of losses as $\epsilon \rightarrow 0_+, \epsilon \tau \rightarrow \infty$.

\begin{table}[t]
\centering
\caption{Datasets description}
\label{tab:datasets}
\begin{tabular}{lcc}
\toprule
Dataset &  \# Examples & \# Features \\ 
\midrule
Appetency\protect\footnotemark & 50000 & 231 \\
\addtocounter{footnote}{-1}
Churn\protect\footnotemark & 50000 & 231 \\
\addtocounter{footnote}{-1}
Upselling\protect\footnotemark & 50000 & 231 \\
Adult\protect\footnotemark        & 48842       & 15 \\  
Amazon\protect\footnotemark &  32769 & 9 \\
Click\protect\footnotemark & 399482 & 12 \\
Epsilon\protect\footnotemark & 500K & 2000 \\
Higgs\protect\footnotemark & 11M & 28 \\
Internet\protect\footnotemark & 10108 & 69 \\
Kick\protect\footnotemark & 72983 & 36 \\
\bottomrule
\end{tabular}
\end{table}

\addtocounter{footnote}{-7}

\footnotetext{\url{https://www.kdd.org/kdd-cup/view/kdd-cup-2009/Data}}
\footnotetext{\url{https://archive.ics.uci.edu/ml/datasets/Adult}}
\footnotetext{\url{https://www.kaggle.com/bittlingmayer/amazonreviews}}
\footnotetext{\url{https://www.kdd.org/kdd-cup/view/kdd-cup-2012-track-2}}
\footnotetext{\url{https://www.csie.ntu.edu.tw/~cjlin/libsvmtools/datasets/binary.html##epsilon}}
\footnotetext{\url{https://archive.ics.uci.edu/ml/datasets/HIGGS}}
\footnotetext{\url{https://www.kdd.org/kdd-cup/view/kdd-cup-2012-track-2}}
\footnotetext{\url{https://www.kaggle.com/c/DontGetKicked}}

\section{Experimental Setup}\label{app:C}

\subsection{Dataset Description}\label{app:C1}

The datasets are listed in Table~\ref{tab:datasets}.

\subsection{Parameter Tuning}\label{app:C2}

For all algorithms, we use the default value 64 for the parameter \textit{border-count} and the default value 0 for \textit{random-strength} ($\rho \ge 0$).

For SGB, we tune \textit{learning-rate} ($\epsilon > 0$), \textit{depth} (the maximal tree depth), and the regularization parameter \textit{l2-leaf-reg}. Moreover, we set \textit{bootstrap-type=Bernoulli}.

For SGLB, we tune \textit{learning-rate}, \textit{depth}, \textit{model-shrink-rate} ($\gamma \ge 0$), and \textit{diffusion-temperature} ($\beta > 0$). 

For all methods, we set \textit{leaf-estimation-method=Gradient} as our main purpose is to compare first order optimization, and use the option \textit{use-best-model=True}.

For tuning, we use the random search (200 samples) with the following distributions:
\begin{itemize}
    \item For \textit{learning-rate} log-uniform distribution over $[10^{-5}, 1]$.
    \item For \textit{l2-leaf-reg} log-uniform distribution over $[10^{-1}, 10^1]$ for SGB and \textit{l2-leaf-reg=0} for SGLB.
    \item For \textit{depth} uniform distribution over $\{6, 7, 8, 9, 10\}$.
    \item For \textit{subsample} uniform distribution over $[0, 1]$.
    \item For \textit{model-shrink-rate} log-uniform distribution over $[10^{-5}, 10^{-2}]$ for SGLB.
    \item For \textit{diffusion-temperature} log-uniform distribution over $[10^2, 10^5]$ for SGLB.
\end{itemize}

\begin{table*}[h]
\caption{Notation used throughout the paper}
\label{tab:notation}
\vskip 0.15in
\begin{center}
\begin{tabular}{cl}
\toprule
Variable & Description  \\
\midrule
$x \in \mathcal{X}$ & Features, typically from $\mathbb{R}^k$ \\
$y \in \mathcal{Y}$ & Target, typically from $\mathbb{R}$ or $\{0, 1\}$ \\
$z \in \mathcal{Z}$ & Prediction, typically from $\mathbb{R}$ \\
$\mathcal{D}$ & Data distribution over $\mathcal{X}\times\mathcal{Y}$\\
$\mathcal{D}_N=\{(x_i, y_i)\}_{i=1}^N$ & I.i.d. samples from $\mathcal{D}$  \\
$L(z,y):\mathcal{Z}\times\mathcal{Y}\rightarrow\mathbb{R}$ & Loss function \\
$\mathcal{L}(f|\mathcal{D})$ & Expected loss w.r.t. $\mathcal{D}$\\
$\mathcal{L}_N(f)$ & Empirical loss \\
$\mathcal{L}_N({F},\gamma)$ & Regularized or implicitly regularized loss\\
$\mathcal{H}$ & Set of weak learners \\
$h^s(x,\theta^s) \in \mathcal{H}$
& Weak learner parameterized by $\theta^s$ \\
$H_{s}: \mathbb{R}^{m_s} \to \mathbb{R}^N$ & Linear operator converting $\theta^{s}$ to $(h^{s}(x_i,\theta^{s}))_{i=1}^N$ \\
$\Theta\in \mathbb{R}^m$ & Ensemble parameters \\
$f_{\Theta}(x):\mathcal{X} \rightarrow \mathcal{Z} $ & Model parametrized by $\Theta\in \mathbb{R}^m$\\
$\tau\in \mathbb{Z}_+$ & Discrete time\\
$t\in [0, \infty)$ & Continuous time\\
$\hat{F}_{\tau}$ & Predictions' Markov Chain $\big(f_{\widehat{\Theta}_\tau}(x_i)\big)_{i=1}^N$ \\ 
$F(t)$ & Markov process $\big(f_{\Theta(t)}(x_i)\big)_{i=1}^N$\\
$V_{\mathcal{B}} \subset \mathbb{R}^N$ & Subspace of predictions of all possible ensembles \\
$p(s|g)$ &  Probability distribution over weak learners’ indices \\
$\Phi_s: \mathbb{R}^N \rightarrow \mathbb{R}^{m_s}$ & Weak learner parameters estimator \\
$P_{s} := H_{s} \Phi_{s}$ & Orthoprojector \\
$P_{\infty} = N\mathbb{E}_{s\sim p(s|\mathbb{0}_N)}P_{s}$ & Implicit limiting preconditioner matrix of the boosting \\
$P = P_{\infty}$ & Symmetric preconditioner matrix \\
$\Gamma= \sqrt{P^{-1}}$ & Regularization matrix \\
$\delta_\Gamma(\gamma)$ & Error from the regularization \\
$p_\beta(\Theta)$ & Limiting distribution of $\widehat{\Theta}_\tau$ \\
$\lambda_*$ & Uniform spectral gap parameter \\
$\epsilon > 0$ & Learning rate \\
$\beta > 0$ & Inverse diffusion temperature \\
$\gamma > 0$ & Regularization parameter \\
$I_m \in \mathbb{R}^{m\times m}$ & Identity matrix \\
$\mathbb{0}_m \in \mathbb{R}^{m}$ & Zero vector \\
$W(t)$ & Standard Wiener process\\
$\phi(x):\mathcal{X}\rightarrow \mathbb{R}^m$ & Feature map, s.t. $f_\Theta(x) = \langle\phi(x), \Theta\rangle_2$ \\
$\Psi := \big[\phi(x_1),\ldots, \phi(x_N)\big]^T \in \mathbb{R}^{N\times m}$ & Design matrix \\
\bottomrule
\end{tabular}
\end{center}
\vskip -0.1in
\end{table*}

\end{document}